\documentclass[a4paper]{article}

\usepackage{setspace} 
\usepackage{booktabs} 

\usepackage[pass]{geometry}
\usepackage{amsmath}
\usepackage{amssymb}
\usepackage{amsthm} 
\usepackage{graphicx}

\usepackage{subfigure}
\newlength{\subfigwidth}
\newlength{\subfigcolsep}
\usepackage{here}

\usepackage{pifont}
\newcommand{\cmark}{\ding{51}}%
\newcommand{\xmark}{\ding{55}}%

\usepackage{algorithm}
\usepackage{algorithmic}
\usepackage{comment}

\newcommand{\nn}{\nonumber\\}
\newtheorem{theorem}{Theorem}

\newtheorem{assumption}{Assumption}

\newcommand{\Expect}{\mathbb{E}}
\newcommand{\Prob}{\mathbb{P}}

\newcommand{\Real}{\mathbb{R}}

\newcommand{\xone}{x}
\newcommand{\xtwo}{z}
\newcommand{\resx}{u}
\newcommand{\Xone}{\mathbf{X}}
\newcommand{\Xtwo}{\mathbf{Z}}

\newcommand{\Xonetrain}{\mathbf{X}_{\mathrm{train}}}
\newcommand{\Xtwotrain}{\mathbf{Z}_{\mathrm{train}}}
\newcommand{\resXtrain}{\mathbf{U}_{\mathrm{train}}}

\newcommand{\Strain}{\mathbf{S}_{\mathrm{train}}}
\newcommand{\Ytrain}{\mathbf{Y}_{\mathrm{train}}}

\newcommand{\dimone}{{d_x}}
\newcommand{\dimtwo}{{d_z}}

\newcommand{\Npos}{n_{s=1}}
\newcommand{\Nneg}{n_{s=0}}
\newcommand{\Ymat}{\mathbf{Y}}
\newcommand{\Smat}{\mathbf{S}}
\newcommand{\resX}{U}
\newcommand{\Cov}[2]{\mathrm{Cov}(#1,#2)}
\newcommand{\CovInv}[2]{\mathrm{Cov}^{-1}(#1,#2)}
\newcommand{\Std}[1]{\sigma_{#1}}
\newcommand{\haty}{\hat{y}}
\newcommand{\Ind}{\mathrm{I}}
\newcommand{\convinprob}{\overset{p}{\to}}
\newcommand{\betas}{\mathbf{B}_s}
\newcommand{\hatbetas}{\hat{\mathbf{B}}_s}
\newcommand{\betaxoneexp}{\beta_x}
\newcommand{\betaxtwoexp}{\beta_z}

\usepackage{authblk}

\author[1]{Junpei Komiyama}
\author[2]{Hajime Shimao}
\affil[1]{University of Tokyo}
\affil[2]{Purdue University}

\begin{document}
\title{Two-stage Algorithm for Fairness-aware Machine Learning}
\maketitle

\begin{abstract}
Algorithmic decision making process now affects many aspects of our lives.
Standard tools for machine learning, such as classification and regression, are subject to the bias in data, and thus direct application of such off-the-shelf tools could lead to a specific group being unfairly discriminated. 
Removing sensitive attributes of data does not solve this problem because a \textit{disparate impact} can arise when non-sensitive attributes and sensitive attributes are correlated. Here, we study a fair machine learning algorithm that avoids such a disparate impact when making a decision. 
Inspired by the two-stage least squares method that is widely used in the field of economics, we propose a two-stage algorithm that removes bias in the training data. The proposed algorithm is conceptually simple. Unlike most of existing fair algorithms that are designed for classification tasks, the proposed method is able to (i) deal with regression tasks, (ii) combine explanatory attributes to remove reverse discrimination, and (iii) deal with numerical sensitive attributes. The performance and fairness of the proposed algorithm are evaluated in simulations with synthetic and real-world datasets.
\end{abstract}


\section{Introduction}
Algorithmic decision making process now affects many aspects of our lives. Emails are spam-filtered by classifiers, images are automatically tagged and sorted, and news articles are clustered and ranked. 
These days, even decisions regarding individual people are being made algorithmically. For example, computer-generated credit  scores are popular in many countries, and job interviewees are sometimes evaluated by assessment algorithms\footnote{https://www.hirevue.com/}. 
However, a potential loss of transparency, accountability, and fairness arises when decision making is conducted on the basis of past data. For example, if a dataset indicates that specific groups based on sensitive attributes (e.g., gender, race, and religion) are of higher risk in receiving loans, direct application of machine learning algorithm would highly likely  result in loan applicants on those groups being rejected. 

In other words, a machine learning algorithm that utilizes sensitive attributes is subject to biases in the existing data.
This could be viewed as an algorithmic version of \textit{disparate treatment} \cite{uslabor}, where decisions are made on the basis of these sensitive attributes. However, removing sensitive attributes from the dataset is not sufficient solution as it has a \textit{disparate impact}.
Disparate impact is a notion that was born in the 1970s. The U.S. Supreme Court ruled that the hiring decision at the center of the Griggs v. Duke Power Co. case \cite{griggs} was illegal because it disadvantaged an application of to a certain race, even though the decision was not explicitly determined based on the basis of race. Duke Power Co. was subsequently forced to stop using test scores and diplomas, which are highly correlated with race, in its hiring decisions. Later on, the U.S. Equal Employment Opportunity Commission \cite{eeoc} stated that disparate impact arises when some group has a less than 80\% selection probability in hiring, promotion, and other employment decisions involving race, sex, or ethnic group, which is so-called the 80\%-rule. Economists have argued that not only is disparate impact ethically problematic, in the labor market, it may also cause inefficiency in social welfare even when the discrimination is based on statistical evidence \cite{norman2003statistical}. 

In this paper, we are interested in a fair algorithm that prevents disparate impact. 
The study of disparate impact is not new in the context of fairness-aware machine learning. As we review in Section \ref{subsec_related}, there are a number of existing studies on this topic. However, there are three major limitations on the existing algorithms intended to alleviate disparate impact:
\begin{itemize}
\item Most of the existing algorithms are built for classification tasks and cannot deal with regression tasks. Arguably, classification is very important. A fair classifier avoids disparate impact in making decisions on admission in universities and hiring by companies. Still, there are tasks that require continuous attributes, such as salaries quoted in a job offer and penalties of criminals. Unfortunately, only a few algorithms Calders et al. \cite{DBLP:conf/icdm/CaldersKKAZ13,berk17} are able to handle regression. Note that data pre-processing algorithms such as Zemel et al. \cite{DBLP:conf/icml/ZemelWSPD13,DBLP:conf/kdd/FeldmanFMSV15} could also be combined with regressors.
\item Direct application of fairness could lead to reverse discrimination. To see this, let us take the example of income \cite{DBLP:conf/icdm/ZliobaiteKC11}. In the Adult dataset \footnote{https://archive.ics.uci.edu/ml/machine-learning-databases/adult/}, women on average have lower incomes than men. However, women in the dataset work fewer hours than men per week on average. A fairness-aware classifier built on the top of this dataset, which equalizes the wages of women and men, leads to a reverse discrimination that makes the salary-per-hour of men smaller than that of women. Such discrimination can be avoided by introducing explanatory attributes and this allows us to make a difference on the basis of the explanatory attributes. In fact, the as in case of Griggs v. Duke Power Co., promoting decisions that cause disparate impacts is not allowed because they are not based on a reasonable measure of job performance, which implies (in some cases) decisions can be fair if they are of reasonable explanatory attributes. Unfortunately, most of the existing studies cannot utilize explanatory attributes.
\item Existing algorithms cannot deal with numerical (continuous) sensitive attributes. Although most sensitive attributes, such as gender, race, and religions are binary or categorical (polyvalent), some sensitive attributes are naturally dealt with in terms of numerical values. For example, the Age Discrimination Act \cite{agedisc} in the U.S. prohibits discrimination in hiring, promotion, and compensation on the basis of age for workers age 40 or above; here, age is a sensitive attribute that is naturally dealt with numerically.  
\end{itemize}

\begin{table*}[b!]
\caption{List of fair estimators and their capabilities. ``Categorical sensitive attrs'' indicates that an algorithm can deal with more than binary sensitive attributes. ``Numeric sensitive attrs'' indicates that an algorithm can deal with continuous sensitive attributes. ``explanatory attrs'' indicates that an algorithm utilizes some attributes that justify the treatment (e.g. the effect of working hours on wages) \cite{DBLP:conf/icdm/ZliobaiteKC11}. The checkmark indicates the capability of the algorithm in the corresponding aspect.}
\begin{small}
\begin{center}
  \begin{tabular}{|c|p{1.7cm}|p{1.7cm}|p{1.7cm}|p{1.7cm}|p{1.7cm}|} \hline
   algorithms & categorical sensitive attrs & numeric sensitive attrs & explanatory attrs & classification & regression \\\hline\hline
Kamiran et al. \cite{kamiran2010} & \xmark & \xmark & \xmark & \cmark & \xmark \\\hline
Zliobaite et al. \cite{DBLP:conf/icdm/ZliobaiteKC11} & \xmark & \xmark & \cmark & \cmark & \xmark \\\hline
Kamishima et al. \cite{DBLP:conf/pkdd/KamishimaAAS12} & \xmark & \xmark & \xmark & \cmark & \xmark \\\hline
Calders et al. \cite{DBLP:conf/icdm/CaldersKKAZ13} & \xmark & \xmark & \cmark & \cmark & \cmark \\\hline
Zemel et al. \cite{DBLP:conf/icml/ZemelWSPD13} & \cmark & \xmark & \xmark & \cmark & \cmark \\\hline
Fish et al. \cite{fish15} & \xmark & \xmark  & \xmark & \cmark & \xmark \\\hline
Feldman et al. \cite{DBLP:conf/kdd/FeldmanFMSV15} & \cmark & \xmark & \xmark & \cmark & \cmark \\\hline
Zafar et al. \cite{DBLP:conf/aistats/ZafarVGG17} & \cmark & \cmark & \xmark & \cmark & \xmark \\\hline
\textbf{This paper}  & \cmark & \cmark & \cmark & \cmark & \cmark \\\hline
  \end{tabular}

\end{center}
\end{small}
\label{tbl_related work}
\end{table*}

Table \ref{tbl_related work} compares our algorithm with the existing ones.

\textbf{Contributions:} Inspired by the econometrics literature, we propose a two-stage discrimination remover (2SDR) algorithm (Section \ref{sec_algorithm}). The algorithm consists of two stages. The first removes disparate impact, and the second is for prediction. The first stage can be considered to be a data transformation that makes the linear classifiers of the second stage fair. 
Theoretical property of 2SDR is analyzed in Section \ref{sec_theory}.
We verified the practical utility of 2SDR by using real-world datasets (Section \ref{sec_experiment}).
We showed that 2SDR is a fair algorithm that (i) performs quite well in not only regression tasks but also classification tasks and (ii) is able to utilize explanatory attributes to improve estimation accuracy. Moreover (iii), it reduces discrimination bias in numeric sensitive attributes, which enables us to avoid other classes of discrimination, such as age discrimination \cite{agedisc}.

\subsection{Related work}
\label{subsec_related}

This section reviews the previous work on fairness-aware machine learning algorithms.
These algorithms can be classified into two categories: 
Algorithms of the first category process datapoints before or after putting them into classifier or regressor. Such an algorithm typically transforms training datasets so as to remove any dependency between the sensitive attribute and target attribute. The advantage of these algorithms is generality: they can be combined with a larger class of off-the-shelf algorithms for classification and regression. Moreover, the transformed data can be considered as a ``fair representation'' \cite{DBLP:conf/kdd/FeldmanFMSV15} that is free from discrimination. The biggest downside of these algorithms that they treat a classifier as a black-box, and as a result, they need to change the datapoints drastically, which tends to reduce accuracy.
Regarding the algorithms of this category, Kamiran et al. \cite{kamiran2010} proposed a data-debiasing scheme by using a ranking algorithm. They were inspired by the idea that the datapoints close to the class borderline are prone to discrimination, and they resample datapoints so as to satisfy fairness constraints. Zliobaite et al. \cite{DBLP:conf/icdm/ZliobaiteKC11} argued that some part of discrimination is explainable by some attributes. They also proposed resampling and relabelling methods that help in training fair classifiers.
Zemel et al. \cite{DBLP:conf/icml/ZemelWSPD13} proposed a method to learn a discrete intermediate fair representation. 
Feldman et al. \cite{DBLP:conf/kdd/FeldmanFMSV15} considered a quantile-based transformation of each attribute.
Hardt et al. \cite{DBLP:conf/nips/HardtPNS16} studied the condition of equalized odds, and provided a post-processing method that fulfills the condition.

Algorithms of the second category directly classify or regress datapoints. Such algorithms tend to perform well in practice since they do not need to conduct explicit data transformation that loses some information. The downside of these algorithms is that one needs to modify an existing classifier for each task. Regarding the algorithms of this approach,
Ristanoski et al. \cite{DBLP:conf/cikm/RistanoskiLB13} proposed a version of support vector machine (SVM), called SVMDisc, that involves a discrimination loss term.
Fish et al. \cite{fish15} shifted the decision boundary of the classical AdaBoost algorithm so that fairness is preserved.
Goh et al. \cite{DBLP:conf/nips/GohCGF16} considered a constrained optimization that satisfies various constraints including the one of fairness. Kamishima et al. \cite{DBLP:conf/pkdd/KamishimaAAS12} proposed prejudice index and proposed a regularizer to reduce prejudice.
Zafar et al. \cite{DBLP:conf/aistats/ZafarVGG17} considered a constrained optimization for classification tasks that maximizes accuracy (resp. fairness) subject to fairness (resp. accuracy) constraint.

Our two-stage approach lies somewhere between the data preprocessing approach and direct approach. The first stage of 2SDR transforms datasets to make the classifier or regressor in the second stage fair. Unlike most data preprocessing algorithms, the transformation of the first stage in 2SDR conducts the minimum amount of transformation that is primarily intended for linear algorithms, and thus, it does not degrade the original information by much. Moreover, any class of linear algorithm can be used in the second stage, and as a result our algorithm can handle more diverse range of tasks and conditions than the existing algorithms can.

Note that other tasks have been considered in the literature of fairness-aware machine learning. To name a few, Kamishima et al. \cite{DBLP:conf/recsys/KamishimaAAS12,DBLP:conf/icdm/KamishimaAAS16} considered methods for removing discrimination in recommendation tasks. Joseph et al. \cite{DBLP:conf/nips/JosephKMR16} considered fairness in the context of online content selection. Bolukbasi et al. \cite{nips:17:03} considered fairness in dense word representation learnt from text corpora.

\section{Problem}
\label{sec_problem}

Each vector in this paper is a column vector and is identified as a $d \times 1$ matrix where $d$ is the dimension of the vector.
Let $n$ be the number of datapoints.
The $i$-th datapoint is comprised of a tuple $(s_i, \xone_i, \xtwo_i, y_i)$, where
\begin{itemize}
\item
$s_i \in \Real^{d_s}$ is the "sensitive" attributes of $d_s$ dimensions that requires special care (e.g., sex, race, and age).
\item
$\xone_i \in \Real^{\dimone}$ is the normal non-sensitive attributes of $\dimone$ dimensions. The difficulty in fairness-aware machine learning is that $\xone_i$ is correlated with $s_i$ and requires to be "fairness adjusted".
\item
$\xtwo_i \in \Real^{\dimtwo}$ is the set of explanatory attributes of $\dimtwo$ dimensions that either are not independent of $s_i$, or not to be adjusted for whatever philosophical reasons. Note that $\xtwo_i$ can be blank (i.e., $\dimtwo=0$) when no explanatory attribute is found.
\item 
$y_i$ is the target attribute to predict. In the case of classification, $y_i \in \{0,1\}$, whereas in the case of regression, $y_i \in \Real$.
\end{itemize}
Note that, unlike most existing algorithms, we allow $s_i$ to be continuous. 
A fairness-aware algorithm outputs $\hat{y}(s, \xone, \xtwo)$, which is an estimator of $y$ that complies with some fairness criteria, which we discuss in the next section.
We also use $\Ymat \in \Real^{n \times 1}, \Xone  \in \Real^{n \times \dimone}, \Xtwo  \in \Real^{n \times \dimtwo}, \Smat  \in \Real^{n \times d_s}$ to denote a sequence of $n$ datapoints. Namely, the $i$-th rows of $\Smat, \Xone, \Xtwo$, and $\Ymat$ are $s_i^\top, \xone_i^\top, \xtwo_i^\top$, and $y_i^\top$, respectively.

\subsection{Fairness criteria}
\label{subsec_fair}

This section discusses fairness criteria that a fairness-aware algorithm is expected to comply with. We consider group-level fairness in the sense of preventing disparate impact \cite{eeoc}, which benefits some group disproportionally.
For ease of discussion, we assume $d_s=1$ and $s$ is a binary\footnote{Although there are several possible definitions, it is not very difficult to extend a fairness measure of binary $s$ to one of a categorical $s$.} or real single attribute.  Note that our method (Section \ref{sec_algorithm}) is capable of dealing with (i) multiple attributes and (ii) categorical $s$ by expanding dummies (i.e., converting them into multiple binary attributes). Let $(s,\xone,\xtwo,y)$ be a sample from the target dataset to make a prediction.  Let $\haty = \haty(\xone, \xtwo)$ be an estimate of $y$ that an algorithm outputs.
For binary $s$ and $\haty$, the P\%-rule \cite{eeoc,DBLP:conf/aistats/ZafarVGG17} is defined as
\begin{equation}
\min_p{\left( \frac{\Prob[\haty=1|s=1]}{\Prob[\haty=1|s=0]}, \frac{\Prob[\haty=1|s=0]}{\Prob[y=1|s=1]} \right)} \ge \frac{p}{100}.
\end{equation}
The rule states that each group has a positive probability at least p\% of the other group. The 100\%-rule implies perfect removal of disparate impact on group-level fairness, and a large value of $p$ is preferred.

For binary $s$ and continuous $\haty$, an natural measure that corresponds to the p\%-rule is the mean distance (MD) \cite{DBLP:conf/icdm/CaldersKKAZ13}, which is defined as:
\begin{equation}
\left| \Expect[\haty|s=1] - \Expect[\haty|s=0] \right|,
\label{eq_md}
\end{equation}
which is a non-negative real value, and a MD value close to zero implies no correlation between $s$ and $y$.
Moreover, Calders et al. \cite{DBLP:conf/icdm/CaldersKKAZ13} introduced the area under the receiver operation characteristic  curve (AUC) between $\haty$ and $s$:
\begin{equation}
\left| \frac{ \sum_{i\in \{1,2,\dots,n\}: s_i=1}\sum_{j\in \{1,2,\dots,n\}: s_j=0} \Ind[\haty_i>\haty_j]}{\Npos \times \Nneg} \right|,
\label{eq_auc}
\end{equation}
where $\Ind[x]$ is $1$ if $x$ is true and $0$ otherwise, and $\Npos$ (resp. $\Nneg$) is the number of datapoints with $s=1$ (resp. $s=0$), respectively. The AUC takes value in $[0,1]$ and is equal to $0.5$ if $s$ shows no predictable effect on $y$.

Moreover, for continuous $s$, we use the correlation coefficient (CC) $| \Cov{s}{\haty} |$ between $s$ and $\haty$ as a fairness measure.
Note that, when $s$ is binary, the correlation is essentially equivalent to MD (Eq. \eqref{eq_md} up to a normalization factor.

\section{Proposed Algorithm}
\label{sec_algorithm}

Here, we start by discussing the two-stage least squares (2SLS), a debiasing method that is widely used in statistics, econometrics, and many branches of natural science that involve observational bias (Section \ref{subsec_2sls}). After that, we describe the two-stage discrimination remover (2SDR) for fairness-aware classification and regression (Section \ref{subsec_2sdr}). Section \ref{subsec_comparepre} compares 2SDR with existing data preprocessing methods.

\subsection{Two-stage least squares (2SLS)}
\label{subsec_2sls}

Consider a linear regression model
\[
 y_i = x_i^\top \beta + \epsilon_i,
\]
where the goal is to predict $y_i \in \Real$ from attributes $x_i \in \Real^{\dimone}$. If the noise $\epsilon_i$ is uncorrelated with $x_i$, an ordinary least square $\hat{\beta}_{\mathrm{OLS}} = (X^\top X)^{-1} X^\top Y$ consistently estimates $\beta$. 
However, the consistent property is lost when $x_i$ is correlated with $\epsilon$: Namely, it is well-known \cite{wooldridge} that, under mild assumption
\[
  \hat{\beta}_{\mathrm{OLS}} \convinprob \beta + \frac{\Cov{x}{\epsilon}}{\sigma_x^2},
\]
where $\Cov{x}{\epsilon}$ is the covariance between $x$ and $\epsilon$. $\sigma_x^2$ is the variance of $x_i$, and the arrow $\convinprob$ indicates a convergence in probability. To remove the bias term, one can utilize a set of additional attributes $z_i$ that are (i) independent of $\epsilon_i$, and (ii) correlate with $x_i$. The crux of 2SLS is to project the columns of $X$ in the column space of $Z$: 
\begin{align}
\hat{X} &= Z (Z^\top Z)^{-1} Z^\top X \nn
\hat{\beta}_{\mathrm{2SLS}} &= (\hat{X}^\top \hat{X})^{-1} \hat{X}^\top Y.
\end{align}
Unlike the OLS estimator, the 2SLS estimator consistently estimates $\beta$. That is,
\[
  \hat{\beta}_{\mathrm{2SLS}} \convinprob \beta.
\]

\subsection{Proposed algorithm: 2SDR}
\label{subsec_2sdr}

Our case considers a classification problem with a fairness constraint (Section \ref{sec_problem}). That is, to estimate the relationship 
\[
y_i = \xone_i^\top \beta + \epsilon_i,
\]
subject to fairness criteria that urges an estimator $\haty_i$ to be uncorrelated to $s_i$ (Section \ref{subsec_fair}).
The main challenge here is that $\xone_i$ is correlated with the sensitive attribute $s_i$, and thus, simple use of the OLS estimator yields a dependency between $\haty_i$ and $s_i$. To resolve this issue, we use $\resX = \Xone - \Smat (\Smat^\top \Smat)^{-1} \Smat^\top \Xone$, which is the residual of $\Xone$ that is free from the effect of $\Smat$, for predicting $\Ymat$. 
In the second stage, we use $\resX$ and $\Xtwo$ to learn an estimator of $\Ymat$ by using an off-the-shelf regressor or classifier. The entire picture of the 2SDR algorithm is summarized in Algorithm \ref{alg_2sdr}. 
One may use any algorithm in the second stage: We mainly intend a linear classifier or regressor for the reason discussed in Section (Theorem \ref{sec_theory}). 
Following the literature of machine learning, we learn the first and the second stage with the training dataset, and use them in the testing data set. 

Note that each column $1,2,\dots,\dimone$ of the first stage corresponds to an OLS estimator, which involves no hyperparameter.
In the case that the second-stage classifier or regressor involves hyperparameters, we perform standard $k$-fold cross validation (CV) on the training dataset to optimize the predictive power of them.

\begin{algorithm}[t!]
\caption{2-Stage Discrimination Remover (2SDR).}
\label{alg_2sdr}
\begin{algorithmic}[1]
\STATE Input: Second stage algorithm $f(\xone, \xtwo)$.
\STATE Using training data $(\Strain,\Xonetrain,\Xtwotrain,\Ytrain)$:
\STATE $\hatbetas \leftarrow (\Strain^\top \Strain)^{-1} \Strain^\top \Xonetrain$. \label{line_fststage}
\STATE $\resXtrain \leftarrow \Xonetrain - \Strain \hatbetas$.
\STATE Train the function $f$ with $(\resXtrain, \Xtwotrain)$. 
\FOR {each data point $(s_i, \xone_i, \xtwo_i, y_i)$ in testdata}
  \STATE Predict $\resx_i^\top \leftarrow \xone_i^\top - s_i^\top \hatbetas$.
  \STATE Predict $\hat{y}_i \leftarrow f(\resx_i, \xtwo_i)$.
\ENDFOR
\end{algorithmic}
\end{algorithm}

\subsection{Comparison with other data preprocessing methods}
\label{subsec_comparepre}

The first stage of 2SDR (Line \ref{line_fststage} of Algorithm \ref{alg_2sdr}) learns a linear relationship between $\Smat$ and $\Xone$. This stage transforms each datapoint by making the second stage estimator free from the disparate impact, so one may view 2SDR as a preprocessing-based method that changes the data representation. This section compares 2SDR with existing methods that transform a dataset before classifying or regressing it. 
At a word, there are two classes of data transformation algorithm: An algorithm of the first class utilizes the decision boundary and intensively resamples datapoints close to the boundary \cite{kamiran2010}. Such an algorithm performs well in classifying datasets, but its extension to a regression task is not straightforward. An algorithm of the second class successfully learns a generic representation that can be used with any classifier or regressor afterward \cite{DBLP:conf/icml/ZemelWSPD13,DBLP:conf/kdd/FeldmanFMSV15}. Such an algorithm tends to lose information at the cost of generality: the method proposed by Zemel et al. \cite{DBLP:conf/icml/ZemelWSPD13} maps datapoints into a finite prototypes, and the one in Feldman et al. \cite{DBLP:conf/kdd/FeldmanFMSV15} conducts a quantile-based transformation, and loses the individual modal structures of the datapoints of $s=0$ and $s=1$. As a result, these methods tend to lose estimation accuracy. Moreover, its extension to a numeric $s$ is non-trivial. 
The first stage in our method can be considered to be a minimum transformation for making linear regression fair and preserves the original data structure. Section \ref{sec_experiment} compares the empirical performance of 2SDR with those of Zemel et al. and Feldman et al. \cite{DBLP:conf/icml/ZemelWSPD13,DBLP:conf/kdd/FeldmanFMSV15}.

\section{Analysis}
\label{sec_theory}

This section analyses 2SDR. 
We first assume the linearity between $\Smat$ and $\Xone$ in the first stage, and derive the asymptotic independence of $\resX$ and $\Smat$ (Theorem \ref{thm_indepdence}). Although such assumptions essentially follow the literature of 2SLS and are reasonable, regarding our aim of achieving fairness, a guarantee for any classes of distribution on $\xone$ and $s$ is desired: Theorem \ref{thm_nocor} guarantees the fairness with a very mild assumption when the second stage is a linear regressor.

\begin{assumption}
Assume the following data generation model where datapoints are i.i.d. drawn:
\begin{equation}
 y_i = \xone_i^\top \beta + \epsilon_i \label{eq_snd}
\end{equation}
and 
\begin{equation}
 \xone_i^\top = s_i^\top \betas + \eta_i^\top
\label{eqn_fstlinear}
\end{equation}
where $\epsilon_i \in \Real$ and $\eta_i \in \Real^{\dimone}$ are mean-zero random variables independent of $s_i$.
\label{asm_tsls}
Moreover, the covariance matrix of $\xone$ is finite and full-rank \footnote{Note that this is a sufficient condition for the ``no perfect collinearity'' condition in Wooldridge \cite{wooldridge}.}.
\end{assumption}

The following theorem states that under Assumption \ref{asm_tsls}, $\resx$ is asymptotically independent of $s$.
\begin{theorem}{\rm (Asymptotic fairness of 2SLS under linear dependency)}
Let $(s_i, \xone_i, \xtwo_i, y_i)$ be samples drawn from the same distribution as the training dataset, and $\resx_i$ is the corresponding residual learnt from the training distribution. 
Under Assumption \ref{asm_tsls}, $\resx_i$ is asymptotically independent of $s_i$. Moreover, if $\xtwo_i$ is independent of $s_i$, $\haty_i$ is asymptotically independent of $s_i$.
\label{thm_indepdence}
\end{theorem}

\begin{proof}
Under Assumption \ref{asm_tsls}, it is well known (e.g., Thm 5.1 in Wooldridge \cite{wooldridge}) that the first-stage estimator is consistent. That is, $\hatbetas \convinprob \betas$ as $n \rightarrow \infty$, from which we immediately obtain $\resx_i \convinprob \eta_i$. By the assumption that $\eta_i$ is independent of $s_i$, $\resx_i$ is asymptotically independent of $s_i$. The independence of $\haty_i$ and $s_i$ follows from the fact that $\haty_i$ is a function of $\resx_i$ and $\xtwo_i$ that are asymptotically independent of $s_i$.
\end{proof}

From Theorem \ref{thm_indepdence}, we see that 2SDR combined with any classifier or regressor in the second-stage is fair (i.e., achieves a p\%-rule for any $p<100$ (resp. any MD $>0$) in classification (resp. regression) with a sufficiently large dataset.
Essentially, Theorem \ref{thm_indepdence} states that if the relation between $\resx$ and $s$ is linear, the first-stage OLS estimator is able to learn the relationship between them, and as a result $\resx$ is asymptotically equivalent to $\eta$, which is the fraction of $\resx$ that cannot be explained by $s$. 

\begin{figure}[t!]
\begin{center}
  \setlength{\subfigwidth}{.90\linewidth}
  \addtolength{\subfigwidth}{-.90\subfigcolsep}
  \begin{minipage}[t]{\subfigwidth}
  \centering
 \subfigure[Distribution of PctUnemployed in the C\&C dataset]{
 \includegraphics[scale=0.6]{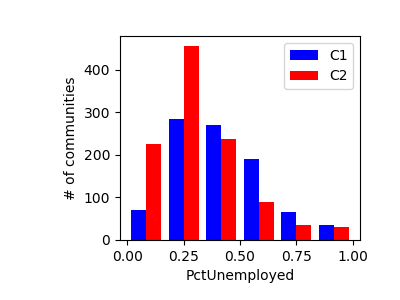}
 }
  \end{minipage}\hfill
  \begin{minipage}[t]{\subfigwidth}
  \centering
  \subfigure[Distribution of age in the Adult dataset]{
   \includegraphics[scale=0.6]{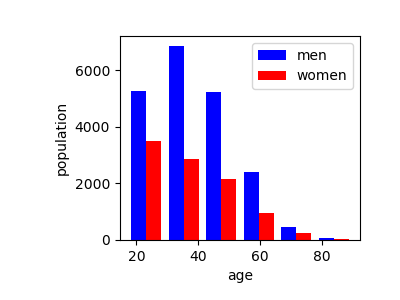}  
  }
  \end{minipage}\hfill
\end{center}
\caption{The first histogram (Figure (a)) shows the percentage of people in the labor force and unemployed (PctUnemployed) in each community in the C\&C dataset, where the horizontal axis is PctUnemployed and vertical axis is the number of corresponding communities. The communities are categorized into the ones with a large portion of black people (C1) and the others (C2). One can see that PctUnemployed in C2 is sharply centered around 0.25, whereas the value in C1 shows a broader spectrum: As a result the variance of PctUnemployed is greatly differ among the two categories. The second histogram (Figure (b)) shows the number of people of different age in the Adult dataset, where the horizontal axis is the age and the vertical axis is the number of people. One can see that not only the variances but also the form of distributions are different between women and men, as majority of the women in the dataset are of the youngest category. The details of these datasets are provided in Section \ref{sec_experiment}.}
\label{fig_hetero}
\end{figure}%

\textbf{Heteroskedasticity in $\xone$:}
As long as Assumption \ref{asm_tsls} holds, $\xone$ is asymptotically independent of $s$. 
However, some of the assumptions may not hold for some attributes in a dataset. 
In particular, Eq. \eqref{eqn_fstlinear} implies that $\xone$ is linear to $s$, and thus, the distribution of $\xone$ conditioned on $s=1$ and $s=0$ is identical after correcting the bias $\Expect[\xone|s=1]-\Expect[\xone|s=0]$ \footnote{For the ease of discussion, let $s$ be binary value here.}. Figure \ref{fig_hetero} shows some attributes where the distribution of $\xone$ is very different among $s=1$ and $s=0$. Taking these attributes into consideration, we would like to seek some properties that hold regardless of the linear assumption in the first stage. The following theorem states that 2SDR has a plausible property that makes $\haty$ fair under very mild assumptions.
\begin{theorem}{\rm (Asymptotic fairness of 2SLS under general distributions)}
Assume that each training and testing datapoint is i.i.d. drawn from the same distribution. Assume that  the covariance matrix of $\xone$ and $s$ are finite and full-rank. Assume that the covariance matrix between $\xone$ and $s$ is finite. Then, the covariance vector $\Cov{s}{\resx} \in \Real^{d_s \times \dimone}$ converges to $0$ in probability as $n \rightarrow \infty$, where $0$ denotes a zero matrix.
\label{thm_nocor}
\end{theorem}

\begin{proof}
Let $(s,\xone, \xtwo, y)$ be a sample from the identical distribution.
The OLS estimator in the first stage is explicitly written as
\[
  \hatbetas = (\Strain^\top \Strain)^{-1} \Strain^\top \Xonetrain,
\]
which, by the law of large numbers, converges in probability to $\CovInv{s}{s} \Cov{s}{\xone}$, where $\CovInv{s}{s} \in \Real^{d_s \times d_s}$ is the inverse of the covariance matrix of $s$, and $\Cov{s}{\xone} \in \Real^{d_s \times \dimone}$ is the covariance matrix between $s$ and $\xone$.
Then, 
\begin{align}
 \Cov{s}{\resx} 
 &= \Cov{s}{\xone} - \Cov{s}{\hatbetas^\top s} \nn
 &\convinprob \Cov{s}{\xone} - \Cov{s}{s} \CovInv{s}{s} \Cov{s}{\xone} \nn
 &= \Cov{s}{\xone} - \Cov{s}{\xone} = 0.
\end{align}
\end{proof}

\begin{figure*}[t!]
\begin{center}
  \setlength{\subfigwidth}{.249\linewidth}
  \addtolength{\subfigwidth}{-.249\subfigcolsep}
  \begin{minipage}[t]{\subfigwidth}
  \centering
 \subfigure[CC as a func. of $n$]{
 \includegraphics[scale=0.35]{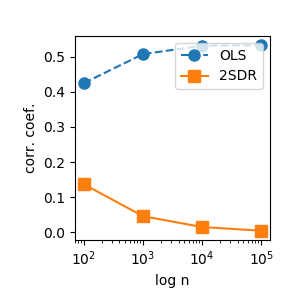}
 }
  \end{minipage}\hfill
  \begin{minipage}[t]{\subfigwidth}
  \centering
  \subfigure[CC as a func. of $\dimone$]{
   \includegraphics[scale=0.35]{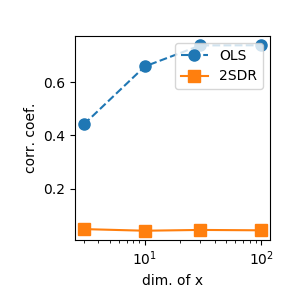}  
  }
  \end{minipage}\hfill
  \begin{minipage}[t]{\subfigwidth}
  \centering
 \subfigure[CC as a func. of $\sigma_{\eta s}$]{
 \includegraphics[scale=0.35]{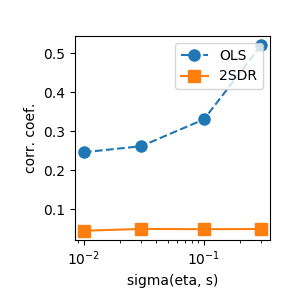}
 }
  \end{minipage}\hfill
  \begin{minipage}[t]{\subfigwidth}
  \centering
 \subfigure[CC as a func. of $\Std{s}$]{
 \includegraphics[scale=0.35]{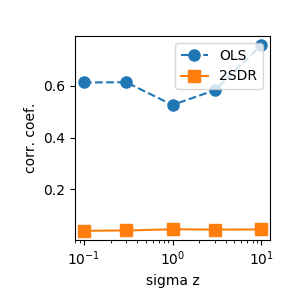}
 }
  \end{minipage}\hfill
  \begin{minipage}[t]{\subfigwidth}
  \centering
 \subfigure[RMSE as a func. of $n$]{
 \includegraphics[scale=0.35]{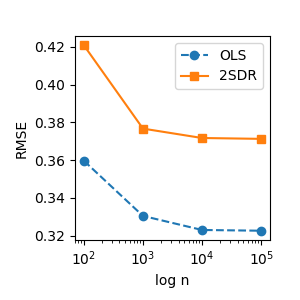}
 }
  \end{minipage}\hfill
  \begin{minipage}[t]{\subfigwidth}
  \centering
 \subfigure[RMSE as a func. of $\dimone$]{
 \includegraphics[scale=0.35]{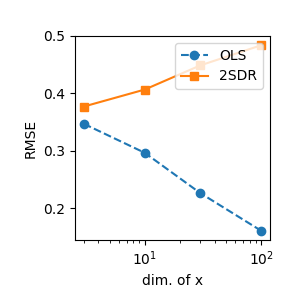}
 }
  \end{minipage}\hfill
  \begin{minipage}[t]{\subfigwidth}
  \centering
 \subfigure[RMSE as a func. of $\sigma_{\eta s}$]{
 \includegraphics[scale=0.35]{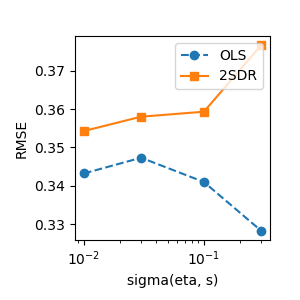}
 }
  \end{minipage}\hfill
  \begin{minipage}[t]{\subfigwidth}
  \centering
 \subfigure[RMSE as a func. of $\Std{s}$]{
 \includegraphics[scale=0.35]{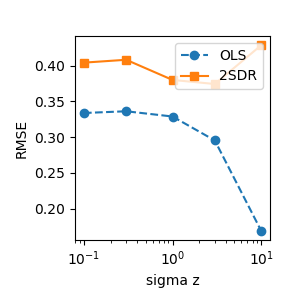}
 }
  \end{minipage}\hfill
\end{center}
\caption{Correlation coefficient (CC) with different parameters (Figures (a)-(d)). Figure (a) is the result with different datasize, Figure (b) is the result with different dimension of $\xone$, Figure (c) is the result with different strength of correlation between $\xone$ and $s$, and Figure (d) is the result with different variance of $s$. One can see that with sufficient large $n$ ($n \ge 1,000$), 2SDR has consistently removes correlation between $s$ and $\haty$. Figures (e)-(h) shows the root mean square error (RMSE) with the same setting as Figures (a)-(d), where RMSE is defined as the squared empirical mean of $(\haty-y)^2$. The larger $\dimone$, $\Cov{\xone}{s}$, or $\Std{s}$ is, the gap of RMSE between 2SDR and OLS is larger. This is because (i) the correlation between $\xone$ and $s$ causes a disparate impact of OLS, (ii) whereas 2SDR, which keeps $\haty$ fair, forces a large bias correction when the correlation is large. For each set of parameters the result is averaged over 100 independent runs.}
\label{fig_synth}
\end{figure*}%

\textbf{Asymptotic fairness of regressor:}
Notice that a linear regressor in the second stage outputs $\haty$ as a linear combination of the elements of $\resx$ and $\xtwo$. Theorem \ref{thm_nocor} implies that a regressor is asymptotically fair in the sense of MD (for binary $s$) or correlation coefficient (for continuous $s$). Unfortunately, it does not necessarily guarantee a fair classification under heteroskedasticity: A linear classifier divides datapoints into two classes by a linear decision boundary (i.e. $\haty$ is whether a linear combination of $\resx$ and $\xtwo$ is positive or negative), and no correlation between $\resx, \xtwo$ and $s$ does not necessarily implies no correlation property between $\haty$ and $s$. Still, later in Section \ref{sec_experiment} we empirically verify the fairness property of 2SDR in both classification and regression.

\textbf{Generalization and finite-time analysis:}
The analysis in this section is very asymptotic and lacks a finite time bound. As OLS is a parametric model, the standard central limit theorem can be applied to obtain the asymptotic properties of the 2SDR estimator: Like the 2SLS estimator, the 2SDR estimator is expected to converge at a rate of $O(1/\sqrt{n})$.

\section{Experiments}
\label{sec_experiment}

In the previous section, we provided results suggesting that 2SDR achieves fairness in an asymptotic sense. To verify the actual performance of 2SDR, we conducted computer simulations. We first describe its results for a synthetic dataset (Section \ref{subsec_synth}), and then describes its results for five real-world datasets (Section \ref{subsec_real}). Our simulation was implemented in Python by using the scikit-learn library\footnote{http://scikit-learn.org/}.
Each of the simulations took from several seconds to several minutes on a modern PC.

\subsection{Synthetic dataset}
\label{subsec_synth}

This section compares 2SDR with the standard OLS estimator on synthetically-generated datasets. 
Each data point $(s_i, \xone_i, \xtwo_i, y_i)$ was generated from the following process, which is the standard assumption in the two-stage regression problem (Section \ref{subsec_2sls}):
\begin{align}
&y_i = \xone_i^\top \betaxoneexp + \xtwo_i^\top \betaxtwoexp + \epsilon  \label{eq_synthsecond} \\
& \xone_i = s_i^\top \beta_s + \eta_i \label{eq_synthfirst} \\
& \xtwo_i \sim N(0,\sigma_\xtwo) \\
& \epsilon \sim N(0,\sigma_\epsilon) \\
&(\eta_i, s_i)\sim N\left(0, \begin{pmatrix}
\sigma_{\eta} & \sigma_{\eta s}\nn
\sigma_{\eta s} & \sigma_s
\end{pmatrix} \right).
\end{align}
Obviously, $\xone_i$ and $s_i$ are correlated, and thus, a naive algorithm that tries to learn Eq. \eqref{eq_synthsecond} suffers a disparate impact, whereas 2SDR tries to untangle this dependency by learning the relationship \eqref{eq_synthfirst} in the first stage.
Unless specified, we set each parameters as follows: $\dimone=\dimtwo=5$ and $d_s=1$. $\sigma_\epsilon=3.0$. $\sigma_\eta$, $\sigma_\xtwo$, and $\sigma_s$ are diagonal matrices with each diagonal entry is $1.0$, and $\sigma_{\eta s}$ is a matrix with each entry is $0.3$. Each entry of $\betaxoneexp$ and $\betaxtwoexp$ are $0.5$, and each entry of $\beta_s$ is $0.2$. The number of datapoint $n$ is set to $1,000$, and $2/3$ (resp. $1/3$) of the datapoints are used as training (resp. testing) datasets, respectively. Figure \ref{fig_synth} shows the correlation coefficient as as measure of fairness and root mean squared error (RMSE) as a measure of prediction power for various values of parameters. 2SDR is consistently fair regardless of the strength of the correlation between $\xone$ and $s$.

\subsection{Real-world datasets}
\label{subsec_real}

This section examines the performance of 2SDR in real-world datasets. The primary goal of Section \ref{subsubsec_regression} and \ref{subsubsec_classification} is to compare the results of 2SDR with existing results. We tried to reproduce the settings of existing papers \cite{DBLP:conf/icdm/CaldersKKAZ13,DBLP:conf/kdd/FeldmanFMSV15} as much as possible. Section \ref{subsubsec_numerics} provides the results with numerical $s$. Section \ref{subsubsec_otherdataset} provides additional results for other datasets. Results with several other settings are shown in Section \ref{subsubsec_otheraspects}.

Table \ref{tbl_datasets} provides statistics on these datasets. Unless explicitly described, we only put the intercept attribute (i.e., a constant $1$ for all datapoints) into $\xtwo$. 
We used OLS in each attribute of the first stage, and OLS (resp. the Ridge classifier) in the regression (resp. classification) of the second stage. Note that the ridge classifier is a linear model that imposes l2-regularization to avoid very large coefficients, which performs better when the number of samples is limited. The strength of the regularizer of the Ridge classifier is optimized by using the ten-fold cross-validation (CV). For binary $s$, (i) we removed the attributes of variance conditioned on $s=0$ or $s=1$ being zero because such a attribute gives a classifier information that is very close to $s$, and (ii) we conducted a variance correction after the first stage that makes the variance of $\resX$ conditioned on $s=1$ and $s=0$ identical.



\begin{table}[t!]
\caption{List of regression or classification datasets. $D$ is the number of binary or numeric attributes (after expanding unordered categorical attributes into dummies (i.e., set of binary dummy attributes)), and $N$ is the number of datapoints. }
\begin{center}
  \begin{tabular}{|c|p{2cm}|c|c|} \hline
   datasets & Regression or Classification & $D$ & $N$ \\\hline\hline
Adult & Classification & 49 & 45,222 \\\hline
Communities \& Crime (C\&C) & Regression & 101 & 1,994 \\\hline
Compas & Classification & 12 & 5,855 \\\hline
German & Classification & 47 & 1,000  \\\hline
LSAC & Classification & 24 & 20,798 \\\hline
  \end{tabular}
\end{center}
\label{tbl_datasets}
\end{table}

\subsubsection{Regression results for C\&C dataset}
\label{subsubsec_regression}

We first show the results of a regression on the Communities and Crime\footnote{http://archive.ics.uci.edu/ml/datasets/communities+and+crime} dataset that combines socio-economic data and crime rate data on communities in the United States.
Following Calders et al. \cite{DBLP:conf/icdm/CaldersKKAZ13}, we made a binary attribute $s$ as to the percentage of black population, which yielded $970$ instances of $s=1$ with a mean crime rate $y=0.35$ and $1,024$ instances of $s=0$ with a mean crime rate $y=0.13$. Note that these figures are consistent with the ones reported in Calders et al. \cite{DBLP:conf/icdm/CaldersKKAZ13}. 
Table \ref{tbl_regression} shows the results of the simulation. At a word, 2SDR removes discrimination while minimizing the increase of the root mean square error (RMSE). 
One can see that in the sense of RMSE, OLS and SEM-MP \cite{DBLP:conf/icdm/CaldersKKAZ13} perform the best, although these algorithms do not comply with the two fairness criteria. On the other hand, 2SDR and SEM-S \cite{DBLP:conf/icdm/CaldersKKAZ13} comply with the fairness criteria, and with 2SDR performing better in the sense of regression among the two algorithms. 
Furthermore, we put two attributes (``percentage of divorced females'' and ``percentage of immigrants in the last three years'') into explanatory attributes $\xtwo$, whose results are shown as ``2SDR with explanatory attrs'' in Table \ref{tbl_regression}. One can see that the RMSE of 2SDR with these explanatory attributes is significantly improved and very close to OLS.

\begin{table}[t!]
\caption{Regression Results. The scores are averaged result over $10$-fold cross validation \cite{DBLP:conf/icdm/CaldersKKAZ13}. The results of SEM-S and SEM-MP are the ones reported in Calders et al. \cite{DBLP:conf/icdm/CaldersKKAZ13}. ``2SDR with explanatory attrs'' shows the result of 2SDR with two explanatory attributes (''FemalePctDiv'',''PctImmigRecent'']) added to $\xtwo$. A smaller MD indicates better fairness, and an AUC close to $0.5$ indicates a very fair regressor. Smaller RMSE indicates better regression accuracy.}
\begin{center}
  \begin{tabular}{|c|c|c|c|} \hline
   Algorithm & MD & AUC & RMSE \\\hline\hline
OLS & 0.22 & 0.85 & 0.14 \\\hline
2SDR & 0.02 & 0.48 & 0.18 \\\hline
2SDR with explanatory attrs & 0.12 & 0.69 & 0.15 \\\hline
SEM-S & 0.01 & 0.50 & 0.20 \\\hline
SEM-MP & 0.17 & 0.76 & 0.14 \\\hline
  \end{tabular}
\end{center}
\label{tbl_regression}
\end{table}

\begin{table}[t!]
\caption{Classification results for the Adult dataset. The column ``Accuracy'' presents the classification accuracy. Unlike OLS, which does not take fairness into consideration, 2SDR complies with the 80\%-rule.}
\begin{center}
  \begin{tabular}{|c|c|c|} \hline
   Algorithm & P\%-rule & Accuracy \\\hline\hline
OLS & 0.30 & 0.84 \\\hline
2SDR & 0.83 & 0.82 \\\hline
  \end{tabular}
\end{center}
\label{tbl_adult}
\end{table}

\begin{table}[t!]
\caption{Classification results for the German dataset. Unlike OLS, 2SDR complies with the 80\%-rule. The result is averaged over $100$ random splits over the training and testing datasets, where two-thirds of the datapoints are assigned to the training dataset at each split.}
\begin{center}
  \begin{tabular}{|c|c|c|} \hline
   Algorithm & P\%-rule & Accuracy \\\hline\hline
OLS & 0.47 & 0.73 \\\hline
2SDR & 0.81 & 0.73 \\\hline
  \end{tabular}
\end{center}
\label{tbl_german}
\end{table}



\subsubsection{Classification result with Adult and German datasets}
\label{subsubsec_classification}

This section shows the result of classification with the Adult and German datasets. The adult dataset is extracted from the 1994 census database, where the target binary attribute indicates whether each person's income exceeds 50,000 dollars or not. German is a dataset that classifies people into good or bad credit risks  \footnote{https://archive.ics.uci.edu/ml/datasets/Statlog+(German+Credit+Data)}.
Following Zemel et al. and Feldman et al. \cite{DBLP:conf/icml/ZemelWSPD13,DBLP:conf/kdd/FeldmanFMSV15}, we used sex (resp. age) in the Adult (resp. German) datasets. Age in the German dataset is binarized into Young and Old at the age of $25$ \cite{DBLP:conf/icdm/CaldersKKAZ13}. Some sparse attributes in Adult are summarized to reduce dimensionality \cite{DBLP:conf/www/ZafarVGG17}.

Let us compare the results shown in Tables \ref{tbl_adult} and \ref{tbl_german} with the ones reported in previous papers. In a nutshell, 2SDR, which complies with the 80\%-rule, outperforms two data preprocessing methods on the Adult dataset, and performs as well as them on the German dataset:
Zemel et al. \cite{DBLP:conf/icml/ZemelWSPD13} reported that their data transformation combined with a naive Bayes classifier has $\sim 80\%$ (resp. $\sim 70\%$) accuracy on the Adult (resp. German) datasets.
Moreover, Feldman et al. \cite{DBLP:conf/kdd/FeldmanFMSV15} reported that their data transformation combined with a Gaussian Naive Bayes classifier had accuracy of $79\sim 80\%$ (resp. $70\sim 76\%$) on the Adult (resp. German) datasets. 
The method by Zemel et al. \cite{DBLP:conf/icml/ZemelWSPD13} coarse-grains the data by mapping them into a finite space, which we think the reason why its performance is not as good as ours. Meanwhile, the quantile-based method by Feldman et al. \cite{DBLP:conf/kdd/FeldmanFMSV15} performed impressively well in the German dataset but not very well in the Adult dataset: In the Adult dataset, it needed to discard most of the attributes that are binary or categorical, which we consider as the reason for the results.

\subsubsection{Numeric $s$}
\label{subsubsec_numerics}

Next, we considered numeric sensitive attributes. 
Table \ref{tbl_numerics} shows the accuracy and correlation coefficient in OLS and 2SDR. On the C\&C dataset, 2SDR reduced correlation coefficient (CC) with a minimum deterioration to its RMSE. In other words, 2SDR was a very efficient at removing the correlation between $\hat{y}$ and $s$.
\begin{table}[t!]
\caption{Results in the case $s$ is median income (C\&C) or age (Adult and German). Note that age was not binarized in the result of this table.} 
\begin{center}
  \begin{tabular}{|c|c|c|c|c|c|c|} \hline
   Algorithm (Dataset) & CC & Accuracy & RMSE \\\hline\hline
OLS (C\&C) & 0.50 & - & 0.14 \\\hline
2SDR (C\&C) & 0.04 & - & 0.17  \\\hline
OLS (Adult) & 0.22 & 0.84 & - \\\hline
2SDR (Adult) & 0.07 & 0.83 & -  \\\hline
OLS (German) & 0.11 & 0.76 & - \\\hline
2SDR (German) & 0.05 & 0.75 & -  \\\hline
  \end{tabular}
\end{center}
\label{tbl_numerics}
\end{table}




\subsubsection{Other Datasets}
\label{subsubsec_otherdataset}

Furthermore, we conducted additional experiments on two other datasets (Table \ref{tbl_others}). The ProPublica Compas dataset \cite{propublica} is a collection of criminal offenders screened in Broward County, Florida during 2013-2014, where $\xone$ is a demographic and criminal record of offenders and $y$ is whether or not a person recidivated within two years after the screening. We set sex as the sensitive attribute $s$. The Law School Admissions Council (LSAC) dataset \footnote{http://www2.law.ucla.edu/sander/Systemic/Data.htm} is a survey among students attending law schools in the U.S. in 1991, where $y$ indicates whether each student passed the first bar examination. We set whether or not the race of the student is black as the sensitive attribute. Similar to the German dataset, we used $2/3$ (resp. $1/3$) of the datapoints as training (resp. testing) datasets, and results are averaged over $100$ runs. The results, shown in Table \ref{tbl_others} implies that 2SDR complies with the 80\%-rule with an insignificant deterioration in classification performance on these datasets.

\begin{table}[t!]
\caption{
Results for the Compas and LSAC datasets. We balanced training data by resampling in LSAC dataset to cope with class inbalance problem. Compas-R is a version of the Compas dataset where predictive attributes are dropped: In this version, we dropped the attributes of the original dataset whose correlation with $y$ was stronger than $0.3$. This significantly reduces the prediction accuracy and fairness of the OLS estimator which tries to utilize the available information as much as possible. Unlike OLS, the fairness of 2SLS does not decrease even if these attributes are dropped.} 
\begin{center}
  \begin{tabular}{|c|c|c|c|c|c|c|} \hline
   Algorithm (Dataset) & P\%-rule & Accuracy \\\hline\hline
OLS (Compas) & 0.59 & 0.73 \\\hline
2SDR (Compas) & 0.92 & 0.73 \\\hline
OLS (Compas-R) & 0.19 & 0.65 \\\hline
2SDR (Compas-R) & 0.93 & 0.65 \\\hline
OLS (LSAC) & 0.21 & 0.75 \\\hline
2SDR (LSAC) & 0.86 & 0.73 \\\hline
  \end{tabular}
\end{center}
\label{tbl_others}
\end{table}

\subsubsection{Other settings}
\label{subsubsec_otheraspects}

This section shows results with several other settings.

\textbf{Multiple $s$:}
Here, we report the result for multiple sensitive attributes. Table \ref{tbl_multis} lists the results for the Adult dataset, where $s$ is sex (binary) and age (numeric).
One can see that (i) 2SDR reduces discrimination for both of sensitive attributes with a very small deterioration on the classification performance, and (ii) the power of removing discrimination is weaker than in the case of applying 2SDR to a single $s$.  
\begin{table}[t!]
\caption{Performance of 2SDR on the Adult dataset where $s$ is (sex, age). We show p\%-rule (resp. correlation coefficient, CC) with respect to sex (resp. age). Both fairness criteria are improved by using 2SDR.} 
\begin{center}
  \begin{tabular}{|c|c|c|c|} \hline
   Algorithm & P\%-rule & CC & Accuracy \\\hline\hline
OLS & 0.30 & 0.22 & 0.84 \\\hline
2SDR & 0.65 & 0.10 & 0.82 \\\hline
  \end{tabular}
\end{center}
\label{tbl_multis}
\end{table}

\textbf{Effect of ordinal transformation:}
The method proposed by Feldman et al. \cite{DBLP:conf/kdd/FeldmanFMSV15} conducts a quantile-based transformation. We have also combined the transformation with 2SDR. Let $\xone_{i,(k)}$ be the $k$-th attribute in $\xone_i$. A quantile-based transformation maps each attribute $\xone_{i,(k)}$ into its quantile rank among its sensitive attributes $s_i$:
\begin{equation}
 \mathrm{Rank}_{i,k} = \frac{\sum_{j\in\{1,2,\dots,n\}:s_i=s_j} \Ind[\xone_{i,(k)}>\xone_{i,(k)}] }{ |\{j\in\{1,2,\dots,n\}:s_i=s_j\}| }.
\label{eq_quantiletrans}
\end{equation}
Feldman et al. \cite{DBLP:conf/kdd/FeldmanFMSV15} showed that the dependence between $\xone_{i,(k)}$ and $s$ can be removed by using such a quantile-based transformation (c.f. Figure 1 in Feldman et al. \cite{DBLP:conf/kdd/FeldmanFMSV15}). 
Table \ref{tbl_ordinal} and \ref{tbl_ordinal_reg} list the results of applying the transformation of Eq. \eqref{eq_quantiletrans} for each non-binary attribute. Applying an ordinal transformation slightly decreased accuracy (or increased RMSE in regression), as it discards the modal information on the original attribute. 
\begin{table}[t!]
\caption{Classification results for the Adult dataset, with or without the ordinal transformation of Eq. \eqref{eq_quantiletrans}. ``With ordinal trans.'' (resp. ``Without ordinal trans.'') indicates an ordinal transformation is conducted (resp. is not conducted) for each attribute. ''cont. only'' indicates that non-numeric attributes in $\xone$ are discarded beforehand.} 
\begin{center}
  \begin{tabular}{|c|c|c|c|c|} \hline
  & \multicolumn{2}{ |c| }{Without ordinal trans.} & \multicolumn{2}{ |c| }{With ordinal trans.} \\ 
  \hline
   Algorithm & P\%-rule & Accuracy & P\%-rule & Accuracy \\\hline\hline
OLS & 0.30 & 0.84 & 0.29 & 0.83 \\\hline
2SDR & 0.83 & 0.82 & 0.82 & 0.82  \\\hline
OLS (cont. only) & 0.22 & 0.81 & 0.06 & 0.80 \\\hline
2SDR (cont. only) & 0.88 & 0.79 & 0.87 & 0.78 \\\hline
  \end{tabular}
\end{center}
\label{tbl_ordinal}
\end{table}
\begin{table}[t!]
\caption{Regression results for the C\&C dataset.} 
\begin{center}
  \begin{tabular}{|c|c|c|c|c|} \hline
  & \multicolumn{2}{ |c| }{Without ordinal trans.} & \multicolumn{2}{ |c| }{With ordinal trans.} \\ 
  \hline
   Algorithm & MD & RMSE & MD & RMSE \\\hline\hline
OLS  & 0.22 & 0.14 & 0.23 & 0.16 \\\hline
2SDR & 0.02 & 0.18 & 0.02 & 0.19 \\\hline
  \end{tabular}
\end{center}
\label{tbl_ordinal_reg}
\end{table}

\textbf{Generalized linear models:}
We also tried logistic regression in the second stage classifier. Logistic regression is a binary classification model that assumes the following relation between the attributes $x_i$ and target $y_i$:
\begin{equation}
 \Prob[y_i=1|x_i] = \frac{1}{1+e^{-x_i^\top \beta}},
 \label{eq_sigmoid}
\end{equation}
where $\beta$ is the model parameter to be learnt. 
Table \ref{tbl_lr} shows the results of classification when we replaced the second-stage classifier with the logistic regression. Compared with a linear model (Ridge classifier), this yielded a lower p\%-rule in the Adult dataset. This fact is consistent with Theorem \ref{thm_nocor}. It states that $\resx$ is asymptotically uncorrelated to $s$: However, a non-linear map such as the sigmoid function in Eq. \eqref{eq_sigmoid} can cause bias between the mapped $\resx$ and $s$.
We should also note that the more involved non-linear second stage classifiers, such as naive Bayes classifiers, support vector machines, and gradient boosting machines, resulted in a significantly lower p\%-rule than logistic regression because of their strong non-linearity.
\begin{table}[t!]
\caption{Classification result of 2SDR combined with logistic regression.}
\begin{center}
  \begin{tabular}{|c|c|c|c|} \hline
   Algorithm & dataset & P\%-rule & Accuracy \\\hline\hline
2SDR & Adult & 0.72 & 0.83 \\\hline
2SDR & German & 0.80 & 0.73 \\\hline
  \end{tabular}
\end{center}
\label{tbl_lr}
\end{table}

\section{Conclusion}

We studied indirect discrimination in classification and regression tasks.
In particular, we studied a two-stage method to reduce disparate impact.
Our method is conceptually simple and has a wide range of potential applications.
Unlike most of the existing methods, our method is general enough to deal with both classification and regression with various settings.
It lies midway between a fair data preprocessing and a fair estimator: It conducts a minimum transformation so that linear algorithm in the second stage is fair.
Extensive evaluations showed that our method complied the 80\%-rule the tested real-world datasets. 

The following are possible directions of future research:
\begin{itemize}
\item \textbf{Other criteria of fairness:} While the disparate impact considered in this paper is motivated by the laws in the United States, the notion of fairness is not limited to disparate impact \cite{berk2017fairness}. To name a few studies, the equalized odds condition \cite{DBLP:conf/nips/HardtPNS16} and disparate mistreatment \cite{DBLP:conf/www/ZafarVGG17} have been considered. Extending our method to other criteria of fairness would be interesting.
\item \textbf{Non-Linear second stage:} In this study, we restricted the second-stage algorithm to be linear. The main reason for doing so is that the first stage in 2SDR is designed to remove the correlation between $\haty$ and $s$, which is very suitable to linear algorithms (Theorem \ref{thm_indepdence} and \ref{thm_nocor}). 
We have also conducted some experiment with generalized linear model in the second stage (Section \ref{subsubsec_otheraspects}), where we observed a inferior fairness than a linear model. 
Extending our work to a larger class of algorithms would boost the accuracy of 2SDR on some datasets where non-linearity is important.
\end{itemize}

\clearpage
\bibliographystyle{plain}
\bibliography{main.bib}

\end{document}